\PassOptionsToPackage{hyperfootnotes=false}{hyperref}
\documentclass[pdflatex,sn-apa]{sn-jnl}
\usepackage{graphicx,booktabs,array,multirow,amsmath,amssymb,amsthm,xurl,rotating}
\usepackage{capt-of}
\usepackage{setspace}
\usepackage{placeins}
\AtBeginDocument{\interdisplaylinepenalty=10000}
\renewcommand{\abstractfont}{\normalfont\normalsize\unboldmath\leftskip=0pt\rightskip=0pt\parfillskip=0pt plus 1fil}
\AddToHook{env/thebibliography/begin}{\interlinepenalty=10000}
\AddToHook{cmd/thebibliography/after}{%
  \def\newblock{\ifvmode\else\hskip .11em plus .33em minus .07em\fi}}
\newtheorem{definition}{Definition}
\newtheorem{assumption}{Assumption}
\newtheorem{proposition}{Proposition}
\renewcommand{\orcidlogo}{\textsuperscript{\scriptsize iD}}
\hypersetup{pdfauthor={{\"O}zkan CANAY},pdftitle={SoftMCC: An MCC-Brier Calibration Bridge for Threshold-Free Model Selection under Class Imbalance},bookmarksdepth=2}
\begin{document}
\title{SoftMCC: An MCC-Brier Calibration Bridge for Threshold-Free Model Selection under Class Imbalance}
\author*[1]{\fnm{\"Ozkan} \sur{CANAY}\,\orcid{https://orcid.org/0000-0001-7539-6001}}
\email{canay@sakarya.edu.tr}
\affil[1]{\orgdiv{Department of Information Systems and Technologies, Faculty of Computer and Information Sciences}, \orgname{Sakarya University}, \orgaddress{\city{Sakarya}, \country{T\"urkiye}}}
\abstract{Model selection for imbalanced binary classification often uses the Matthews
correlation coefficient (MCC), but its thresholding step makes validation rankings depend on the chosen cutoff. SoftMCC is a post-training MCC validation framework built on established probability-valued confusion counts, coupling an MCC-specific
calibrated identity with a tie-aware, shared-pool selection protocol. Its core
score is a covariance-normalized probability-label association, reduces exactly
to MCC for hard predictions, and is Pearson-bounded. Under perfect population
calibration it equals the Brier skill score with identical candidate ordering,
which is the calibrated case of Murphy's partition of the Brier score.
Outside that regime the gap does not identify calibration error. Across 18
settings with 12 duplicate-safe grouped repeats, SoftMCC attains the best
stability mean rank ($2.31$) and highest mean tie-corrected Kendall's $W$
($0.659$), with a significant Friedman test ($p=0.004$). Nemenyi analysis
separates it from area under the precision-recall curve (AUPRC) and MCC@0.5, and 14-source-family sensitivity retains
both separations. For selected-model utility, none of six
prespecified paired comparisons is significant before or after Holm correction (smallest raw $p=0.093$), and the dataset-level test is
not significant ($p=0.746$). Label permutation lowers mean $W$ to $0.092$.
Temperature scaling shifts SoftMCC rankings (mean Spearman $0.851$) whereas
rank-based and threshold-optimized metrics remain invariant. SoftMCC is a
calibration-sensitive MCC-family selector with bounded stability evidence and
inconclusive utility differences.}


\begin{singlespace}
\maketitle
\end{singlespace}

\section{Introduction}
\label{sec:intro}

Binary classification under class imbalance is central to fraud detection,
intrusion detection, medical diagnosis, and risk prediction, where rare events
make model-selection errors consequential. Accuracy can be dominated by the
majority class, while summaries such as $F_1$ use only part of the confusion
matrix. The Matthews correlation coefficient (MCC) instead uses all four cells and is the $\phi$ correlation between
predicted and true binary labels
\citep{matthews1975,chicco2020advantages,chicco2021reliable,
chicco2023replace}. This interpretation keeps positive and
negative errors visible. Metric choice can nevertheless change which classifier
is preferred on the same data
\citep{ferri2009experimental,gaudreault2024empirical,delacruz2025performance},
and imbalance can alter the behavior and information content of
confusion-matrix measures
\citep{luque2019impact,mullick2020appropriateness,vanzyl2025analysis}.

Threshold-free measures do not remove this choice; they encode different
targets \citep{hernandez2012unified,shirdel2026framework}. The area under the receiver operating characteristic (ROC) curve (AUROC) summarizes pairwise ranking across thresholds, the area under the precision-recall curve (AUPRC) focuses on positive-class retrieval under skew, and proper scores assess probability
quality \citep{gneiting2007strictly,silva2023calibration}. Hard MCC instead evaluates one binary operating point. None of these
roles makes another metric redundant. The relevant question is whether a
probability-valued MCC can serve a defined validation-time role while preserving
an interpretable connection to hard MCC and calibrated probability assessment.

Classical MCC is nevertheless computed after a decision threshold has turned a
probability into a hard label. In practice the threshold may be fixed at $0.5$,
tuned on a validation split, or selected implicitly by an operating convention.
Modern classifiers usually expose probability scores whose calibration is
itself a modeling issue
\citep{niculescu2005predicting,zadrozny2002transforming,guo2017calibration,
vancalster2016calibration,silva2023calibration,dong2025survey}. Consequently, threshold search, prevalence, and
probability scale can change both the MCC value and the candidate ranking it
induces. Validation-time model selection then becomes entangled with a separate
deployment decision about the operating threshold.

A validation metric often selects among models, hyperparameters, calibration
schemes, or preprocessing variants before the test set is touched. Its
assessment should therefore cover both the reproducibility of the validation
ranking and the held-out utility of the selected model. Stability alone is
insufficient if the resulting choice loses test utility. This study measures
both endpoints and reports calibration sensitivity explicitly.

Replacing hard confusion-matrix counts with probability mass offers one
response. Probability-valued confusion matrices already support post-training
classifier evaluation through probabilistic confusion entropy and
confidence-weighted precision, recall, and $F_1$
\citep{wang2013pcen,yacouby2020probabilistic}. A general probabilistic-matrix
formulation also covers confusion-matrix measures including MCC
\citep{aguilarruiz2024certainty}.
Soft-count MCC has separately appeared as a differentiable training loss
\citep{abhishek2021mcc}, while AnyLoss \citep{han2024anyloss} and the soft-set
formulation of \citet{tsoi2022bridging} cover the broader optimization role.

SoftMCC builds on this shared kernel with an MCC-specific theory and selection
protocol. The formal analysis establishes its covariance form, exact hard-label
reduction, and a one-way identity with the Brier skill score under perfect
population calibration. The identity is the calibrated case of the partition of
the Brier score by \citet{murphy1973partition}, written for the soft-count MCC.
The empirical protocol evaluates tie-aware ranking
reproducibility and selected-model hard-MCC utility on a shared candidate pool,
reports stability at both the setting and the source-family level together
with paired selected-utility comparisons, and quantifies calibration sensitivity
directly. SoftMCC is therefore a theory-to-selection bridge built on
established probability-valued confusion matrices rather than a new
confusion-matrix construction. Because the identity is one-way, equality with
the Brier skill score does not certify calibration and disagreement does not
quantify calibration error. These properties motivate three research
questions.

\textit{RQ1}: Does the soft-count MCC formula have a simple statistical
interpretation, and does it remain within the MCC family when predictions
become hard labels?

\textit{RQ2}: Does a SoftMCC-based validation ranking retain test MCC
relative to established selectors such as MCC@best, $F_1$@best, AUROC, AUPRC,
and the Brier score?

\textit{RQ3}: How reproducible are SoftMCC-induced candidate rankings
across repeated splits, how do they change under monotone calibration shift,
and can label permutation together with the $\rho$ and $\kappa$ factorization
separate label association from probability sharpness as a mechanism?

The empirical design covers 18 binary tabular settings with varied class
prevalence, each evaluated with 12 grouped repeated splits. Exact feature-label
duplicate groups remain
within one split, all selectors use identical cached candidate predictions, and
ties receive midranks. Because several benchmark settings are derived from
related source datasets, the primary setting-level analysis is accompanied by a
conservative source-family analysis that averages related variants before
ranking. This second unit of analysis defines the defensible generalization
boundary for the stability claim.
The evidence covers calibrated, imbalanced binary tabular tasks. Cost-sensitive
weighting, reference-prevalence adjustment, net benefit, temporal drift, and
image, text, or multiclass applications require separate protocols.

Against this background, the study makes four contributions:
\begin{itemize}
  \item A post-training MCC framework that couples a probability-valued core
        score to a threshold-free validation rule and evaluates each selected
        model by the same held-out hard-MCC endpoint.
  \item A closed-form characterization covering covariance form, exact
        hard-label reduction, a Pearson bound, the $\rho$ and $\kappa$
        decomposition of label association and probability sharpness, and
        calibration dependence.
  \item A one-way bridge under perfect population calibration, obtained as the
        calibrated case of the Murphy partition of the Brier score: SoftMCC equals
        the Brier skill score and induces the same candidate ordering, while a
        gap from this identity is not treated as an identifiable calibration
        error.
  \item An 18-setting, duplicate-safe study of ranking reproducibility and
        selected-model utility. Conservative source-family aggregation retains
        both stability separations, while the six paired utility comparisons,
        none of them significant, and calibration sensitivity constrain the
        interpretation.
\end{itemize}

The remainder of the paper is organized as follows. Section~\ref{sec:related}
positions SoftMCC relative to MCC-based evaluation, metric losses, probability
calibration, and statistical comparison practice. Section~\ref{sec:method}
defines the core score and its properties. Sections~\ref{sec:design} and
\ref{sec:results} describe the protocol and evidence. Sections~\ref{sec:discussion} and \ref{sec:conclusion} interpret the evidence and state its limits.
\section{Related Work}
\label{sec:related}

MCC \citep{matthews1975} is the Pearson $\phi$ coefficient between predicted and
true binary labels. It has been promoted as more informative than accuracy,
$F_1$, balanced accuracy, markedness, and several other confusion-matrix
summaries for imbalanced two-class evaluation
\citep{chicco2020advantages,chicco2021reliable,chicco2023replace}.
Broader surveys of imbalanced learning and evaluation
stress that metric choice must be tied to prevalence, threshold policy, and
operating assumptions \citep{he2009learning,krawczyk2016learning}, and recent
metric-behavior analyses under imbalance reach the same conclusion
\citep{vanzyl2025analysis,gaudreault2024empirical}.
Cost-sensitive work similarly shows that the chosen metric can change the
preferred classifier when error costs or class priors move
\citep{elkan2001cost,shirdel2026framework}.

\citet{luque2019impact} systematically analyze
binary confusion-matrix metrics under imbalance and identify MCC as the
preferred option when classification errors as well as successes must be
represented. \citet{mullick2020appropriateness} complement that analysis with
formal distortion-resilience and information-retention conditions for binary
and multiclass performance indices.
Recent empirical studies sharpen this point for imbalanced classification by
showing that performance-measure choice can alter model assessment and model
selection, with MCC and precision-recall-based summaries often behaving more
reliably than accuracy or ROC-only summaries under skew
\citep{gaudreault2024empirical,delacruz2025performance}.
Prior MCC work, including classifier design with MCC objectives
\citep{boughorbel2017optimal}, still treats MCC as a hard-label statistic at a
fixed or tuned threshold.

The evaluation-metric literature also warns that threshold-free summaries answer
different questions. ROC analysis summarizes rankings across possible thresholds
\citep{fawcett2006roc}, while precision-recall analysis is often
more informative under severe skew \citep{davis2006relationship,
saito2015precision,flach2015precision}. Cost curves and H-measure-style
arguments make the cost distribution explicit instead of hiding it inside a
single threshold choice \citep{drummond2006cost,hand2009hmeasure,
hernandez2012unified}, and a recent reward-matrix framework maps common binary
metrics to the cost-benefit ratios they implicitly assume
\citep{shirdel2026framework}. Because different metrics encode different loss
surfaces and prevalence assumptions, empirical rankings can change when the
performance measure changes \citep{ferri2009experimental,hernandez2012unified},
and current guidance recommends pairing a metric with a comparison protocol
and statistical test that match the supervised-learning task
\citep{rainio2024evaluation}. These lines motivate evaluating a
probability-valued MCC as a model-selection score with a clearly stated
operating role alongside AUROC, AUPRC, and MCC@best.

Probability-valued confusion matrices have an established post-training
evaluation lineage. \citet{wang2013pcen} aggregate class-probability mass to
derive probabilistic confusion entropy, compare it with mean absolute error,
mean squared error, and multiclass area under the curve (AUC) variants, and examine classifier choice.
\citet{yacouby2020probabilistic} derive confidence-weighted precision, recall,
and $F_1$ from the probabilistic matrix. They distinguish these measures from
calibration and recommend calibration measures or proper scores when calibrated
probabilities are required.

In a 2024 arXiv preprint,
\citet{aguilarruiz2024certainty} gives the general form
$CM^{\star}=T^{\mathsf T}Q$, covering probability-valued counterparts of
confusion-matrix measures including MCC, and evaluates a certainty ratio from
certainty and uncertainty components across classifiers and datasets.
\citet{aguilarruiz2026quantifying} later decomposes this probability-mass
matrix into decisive and residual components and compares probability-aware
with hard-label scores on 27 public benchmark datasets.
These studies establish the probabilistic-count kernel, generic metric
substitution, and probability-aware classifier choice. SoftMCC connects this
established evaluation kernel to proper scoring through the calibrated
MCC-Brier skill-score identity and identical population candidate ordering,
and its shared-candidate protocol evaluates tie-aware ranking reproducibility
together with selected-model hard-MCC utility. The contribution therefore
concerns MCC-specific theory and selection characterization built on
established probability-valued counts.

\begin{table}[tbp]
\centering
\caption{Adjacent work separates probability-valued confusion evaluation, metric losses, calibration, and model-selection comparison.}
\label{tab:recent-work}
\small
\setlength{\tabcolsep}{3pt}
\begin{tabular*}{\textwidth}{@{\extracolsep{\fill}}p{0.25\textwidth}p{0.20\textwidth}p{0.18\textwidth}p{0.26\textwidth}@{}}
\toprule
Reference & Method role & Data or setting & Metric or finding role\\
\midrule
\citet{wang2013pcen} & probabilistic confusion matrix & multiclass classifiers & probabilistic confusion entropy\\
\citet{yacouby2020probabilistic} & probabilistic confusion matrix & natural language processing classifiers & confidence-weighted precision, recall, and $F_1$\\
\citet{abhishek2021mcc} & soft MCC loss & segmentation & differentiable MCC training\\
\citet{tsoi2022bridging} & soft sets & binary classifiers & training-evaluation alignment\\
\citet{chicco2023replace} & hard MCC evaluation & binary tasks & MCC versus AUROC\\
\citet{aguilarruiz2024certainty} & probabilistic confusion matrix & 21 classification datasets & generic probabilistic measure transform and certainty ratio\\
\citet{vanzyl2025analysis} & metric sensitivity analysis & model-agnostic analytical setting & imbalance sensitivity and normalized variants\\
\citet{dong2025survey} & calibration survey and comparison & deep classifiers under class imbalance & 60-method taxonomy and evaluation guidance\\
\citet{delacruz2025performance} & metric comparison & imbalanced binary data & model-selection behavior of common metrics\\
\citet{fonseca2026high} & balanced calibration metric & transformer text classifiers & Balanced Brier Score for high-effectiveness regimes\\
\citet{lin2026probability} & posterior correction & simulations and eight imbalanced datasets & calibration repair after class rebalancing\\
\citet{shirdel2026framework} & reward-matrix framework & analytical binary setting & cost-benefit ratios for common metrics\\
\bottomrule
\end{tabular*}
\end{table}

The same kernel has a separate optimization lineage. \citet{abhishek2021mcc} use
soft-count MCC as a loss to train segmentation networks, and related
differentiable metric work includes Dice-family losses
\citep{sudre2017generalised}, focal loss \citep{lin2017focal}, soft-set metric
formulations \citep{tsoi2022bridging}, and AnyLoss-style transformations of
classification metrics into losses \citep{han2024anyloss}. SoftMCC instead
leaves candidate training unchanged and operates on calibrated predictions
after training. Training losses are therefore not used as baselines, because
they change the optimization objective, whereas the decision a validation
score must support here is validation-set model selection rather than direct
loss minimization or final deployment thresholding. The work most directly
adjacent to this manuscript, including this optimization lineage, is compared
in Table~\ref{tab:recent-work}, which separates method, data, metric, and
finding roles because the contribution crosses probability-valued evaluation,
calibration, metric interpretation, and model selection.

Proper scoring rules such as the Brier score \citep{brier1950verification}
reward calibrated probability assessments rather than hard-label association or
thresholded classification \citep{gneiting2007strictly}, and calibration methods such as
Platt scaling, isotonic regression, beta calibration, and modern neural-network
temperature scaling are used to recalibrate probability estimates
\citep{niculescu2005predicting,zadrozny2002transforming,
kull2017beta,guo2017calibration}. Calibration reviews and risk-model guidance
stress that probability scores are decision inputs, not only ranking devices
\citep{vancalster2016calibration,silva2023calibration}.
Recent calibration work under class imbalance further reports that
minority-class skew can make posterior-probability calibration more difficult
\citep{dong2025survey}, and that rebalancing can distort posterior estimates,
for which monotone post-hoc corrections have recently been studied on
simulations and eight real imbalanced datasets \citep{lin2026probability}.

SoftMCC is threshold-free but depends on probability magnitudes. Under perfect
population calibration it equals the Brier skill score. This equality is the
calibrated case of the partition of the Brier score into reliability,
resolution, and uncertainty \citep{murphy1973partition}, in which the
reliability term vanishes and the skill score equals resolution divided by
uncertainty. \citet{hernandez2012unified} derive a continuous version of this
decomposition and relate the Brier score to expected loss under threshold
choice. Equality does not imply
calibration, and the gap does not identify calibration error. This algebraic
result differs from comparisons of hard-label MCC and Brier as competing
evaluation metrics \citep{chicco2021mccbrier}. Section~\ref{sec:method}
formalizes the distinction.

Benchmark construction is another source of evidence risk. Hyperparameter
search can alter empirical comparisons when candidate pools are small or unevenly
tuned \citep{bergstra2012random,probst2019tunability}. Benchmark suites are
useful because they make such comparisons less dependent on a single dataset
\citep{olson2017pmlb}. Statistical tests over repeated splits also need care
because standard paired tests can be optimistic when resampling dependence is ignored
\citep{nadeau2003inference,rainio2024evaluation}. The
present protocol therefore uses a shared candidate pool, paired nonparametric
summaries, multi-dataset rank tests, and a prespecified expansion to 18 settings.
Related-setting dependence remains a limitation.

Statistical comparison is the final context because the empirical question is
metric-induced model selection, not training accuracy. Paired comparisons across
repeated splits require uncertainty summaries that respect dependence within a
dataset. Wilcoxon signed-rank tests and bias-corrected and accelerated (BCa) bootstrap intervals summarize paired differences, while Cliff's dominance effect size adds a descriptive distributional view
\citep{wilcoxon1945individual,cliff1993dominance,efron1979bootstrap}. For
multi-dataset rankings, the Friedman test and Nemenyi critical difference (CD)
procedure remain common baselines, with extensions and cautions discussed in the
classifier-comparison literature \citep{demsar2006statistical,benavoli2017time}.
This manuscript uses those tools conservatively and reports
non-significant selection differences as similar observed
selected-model utility, not as mathematical equality.

\section{Methodology}
\label{sec:method}

The complete study flow is summarized in Figure~\ref{fig:theory-map}, which
connects the duplicate-safe split and calibration protocol to the SoftMCC core,
the one-way Brier bridge, metric-induced model selection, and the empirical evidence roles used to bound the contribution.

\begin{figure}[!htbp]
\centering
\begin{minipage}{\textwidth}
\centering
\includegraphics[width=\textwidth]{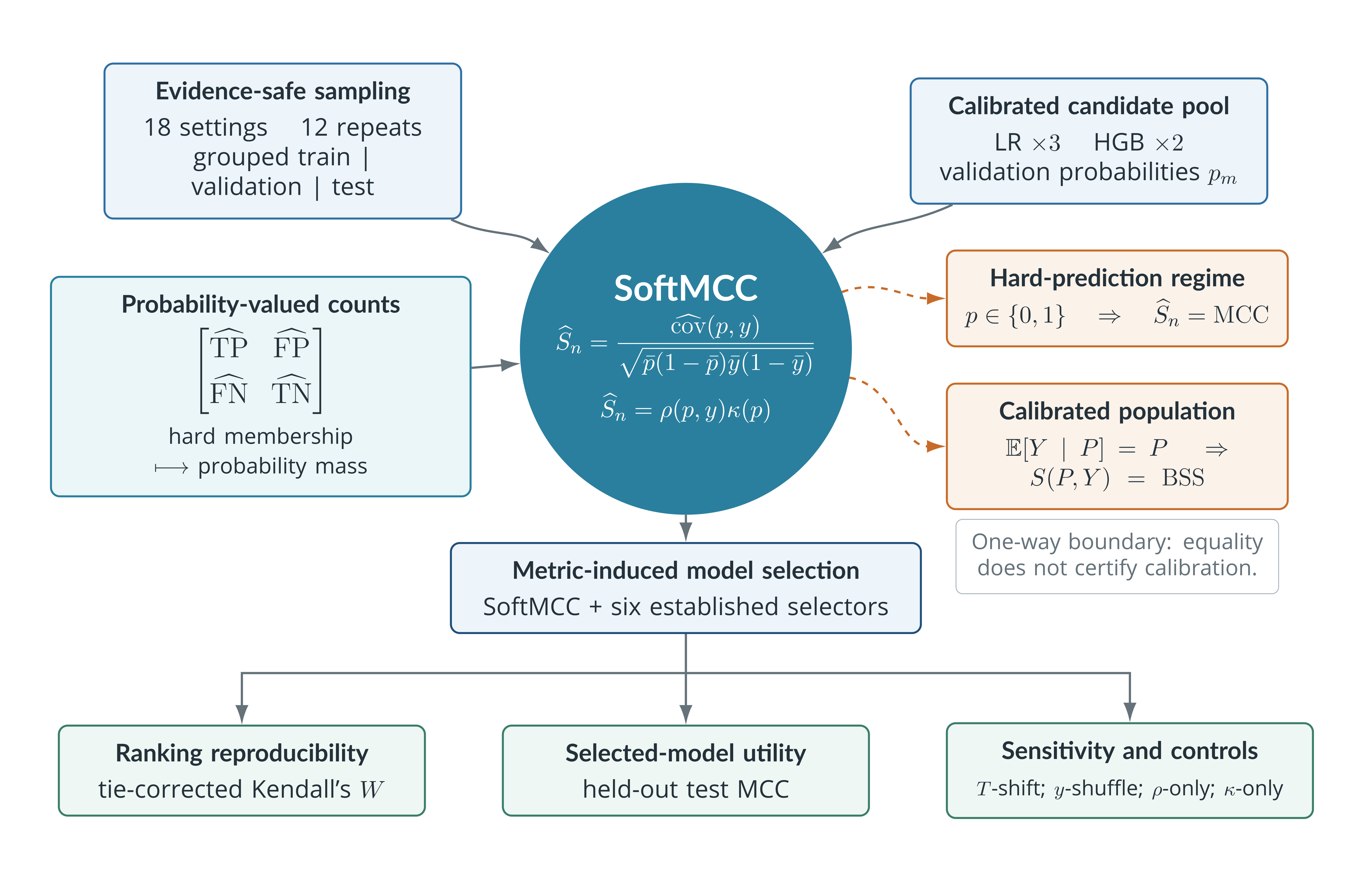}
\captionof{figure}{The duplicate-safe SoftMCC study design links calibrated candidate probabilities to the probability-valued MCC core, its one-way Brier bridge, and evidence on ranking reproducibility, selected-model utility, calibration sensitivity, and mechanism controls.}
\label{fig:theory-map}
\end{minipage}
\end{figure}

The map abbreviates logistic regression as LR, histogram-based gradient
boosting as HGB, and the Brier skill score as BSS; TP, FP, FN, and TN denote
true positives, false positives, false negatives, and true negatives.
Within the map, the two conditional regimes on the right hold only under their
stated conditions, while the three evidence roles at the bottom are measured on
finite benchmark samples. Formal analysis connects the probability-valued core
score to hard MCC, Pearson correlation, and the Brier skill score, then defines
its model-selection protocol.

\subsection{Notation and admissible setting}
Let $\mathcal D_n=\{(y_i,p_i)\}_{i=1}^n$ denote an evaluation sample, where
$y_i\in\{0,1\}$ is the observed label and $p_i\in[0,1]$ is the calibrated
predicted probability assigned to the positive class. Define
\begin{align}
\bar p&=\frac{1}{n}\sum_{i=1}^n p_i,&
\bar y&=\frac{1}{n}\sum_{i=1}^n y_i, \nonumber\\
\widehat{\mathrm{cov}}(p,y)
&=\frac{1}{n}\sum_{i=1}^n(p_i-\bar p)(y_i-\bar y).
\label{eq:sample-moments}
\end{align}
\begin{assumption}[Non-degenerate evaluation marginals]
All theoretical statements below are for $0<\bar p<1$ and $0<\bar y<1$. When a
validation split has a degenerate predicted or label marginal, MCC-family scores
are not defined as correlations. The implementation uses a small numerical guard
only to avoid division by zero, not to change the target score.
\end{assumption}

\subsection{Soft confusion counts and core sample score}
The SoftMCC framework keeps the four-cell structure of MCC but uses a
probability-valued core score after candidate predictions are available. Define
\begin{align}
\widehat{\mathrm{TP}}&=\sum_{i=1}^n p_i y_i,&
\widehat{\mathrm{FP}}&=\sum_{i=1}^n p_i(1-y_i), \nonumber\\
\widehat{\mathrm{FN}}&=\sum_{i=1}^n (1-p_i)y_i,&
\widehat{\mathrm{TN}}&=\sum_{i=1}^n (1-p_i)(1-y_i).
\label{eq:soft-counts}
\end{align}

\begin{definition}[SoftMCC core sample score]
Within SoftMCC, the sample core score is the MCC formula evaluated on the soft
counts in Eq.~\eqref{eq:soft-counts}:
\begin{equation}
\widehat S_n
=\frac{\widehat{\mathrm{TP}}\widehat{\mathrm{TN}}
-\widehat{\mathrm{FP}}\widehat{\mathrm{FN}}}
{\sqrt{\widehat D}},
\label{eq:softmcc}
\end{equation}
where
\begin{align}
\widehat D
&=(\widehat{\mathrm{TP}}+\widehat{\mathrm{FP}})
(\widehat{\mathrm{TP}}+\widehat{\mathrm{FN}}) \nonumber\\
&\quad\times
(\widehat{\mathrm{TN}}+\widehat{\mathrm{FP}})
(\widehat{\mathrm{TN}}+\widehat{\mathrm{FN}}).
\label{eq:soft-denominator}
\end{align}
\end{definition}

These soft counts allocate each observation's probability mass to the
confusion-matrix cells according to its observed class. The denominator in
Eq.~\eqref{eq:soft-denominator} normalizes the resulting association using
the predicted and observed class marginals. The score in
Eq.~\eqref{eq:softmcc} therefore retains the MCC construction while accepting
probabilities directly, without a classification threshold. In the evaluation
protocol of Section~\ref{sec:protocol}, it is computed from validation
predictions to compare candidate models before their selected-model utility
is assessed on held-out test data.

\subsection{Formal properties}
Under non-degenerate marginals, the following properties connect the core score
to hard MCC, Pearson correlation, calibration, and Brier scoring.

\begin{proposition}[Closed form]
The core score in Eq.~\eqref{eq:softmcc} equals
\begin{equation}
\widehat S_n
=\frac{\widehat{\mathrm{cov}}(p,y)}
{\sqrt{\bar p(1-\bar p)\,\bar y(1-\bar y)}}.
\label{eq:closed}
\end{equation}
\end{proposition}

\noindent\emph{Proof.}
Let $S_p=\sum_i p_i$, $S_y=\sum_i y_i$, and $S_{py}=\sum_i p_i y_i$.
Substitution into Eq.~\eqref{eq:soft-counts} gives
\begin{align*}
\widehat{\mathrm{TP}}&=S_{py},&
\widehat{\mathrm{FP}}&=S_p-S_{py},\\
\widehat{\mathrm{FN}}&=S_y-S_{py},&
\widehat{\mathrm{TN}}&=n-S_p-S_y+S_{py}.
\end{align*}
Using the sample moments in Eq.~\eqref{eq:sample-moments}, the MCC numerator
becomes
\[
\widehat{\mathrm{TP}}\widehat{\mathrm{TN}}
-\widehat{\mathrm{FP}}\widehat{\mathrm{FN}}
=nS_{py}-S_pS_y
=n^2\,\widehat{\mathrm{cov}}(p,y).
\]
The four marginals are $S_p$, $S_y$, $n-S_p$, and $n-S_y$, so the denominator
in Eq.~\eqref{eq:soft-denominator} is
\[
\sqrt{S_pS_y(n-S_p)(n-S_y)}
=n^2\sqrt{\bar p(1-\bar p)\bar y(1-\bar y)}.
\]
Dividing numerator by denominator gives Eq.~\eqref{eq:closed}. Thus SoftMCC is a
correlation-type statistic between probabilities and labels, normalized on the
probability side by the Bernoulli variance $\bar p(1-\bar p)$ instead of by the
empirical variance of the probability vector.

\begin{definition}[Population SoftMCC core score]
For a random pair $(P,Y)$ with $P\in[0,1]$, $Y\in\{0,1\}$,
$\mu=\mathbb E[P]$, $\pi=\mathbb E[Y]$, $m_{PY}=\mathbb E[PY]$, and nonzero
denominator, define
\begin{equation}
S(P,Y)=
\frac{m_{PY}-\mu\pi}
{\sqrt{\mu(1-\mu)\,\pi(1-\pi)}}.
\label{eq:population-functional}
\end{equation}
\end{definition}

\begin{proposition}[Reduction to MCC]
\label{prop:reduction-mcc}
If every $p_i\in\{0,1\}$, $\widehat S_n$ equals classical MCC.
\end{proposition}

\noindent\emph{Proof.}
For binary $p$, the soft counts are the usual hard confusion-matrix counts,
$\bar p(1-\bar p)=\mathrm{var}(p)$ and Eq.~\eqref{eq:closed} equals the Pearson
$\phi$ coefficient between the two binary vectors, which is the classical MCC.

\begin{proposition}[Relation to Pearson correlation and bounded range]
\label{prop:pearson-bound}
For $p_i\in[0,1]$ with nonzero $\mathrm{var}(p)$,
\[
|\widehat S_n|\le |\rho(p,y)|\le1,
\]
where $\rho(p,y)$ is the Pearson correlation between the probability and label
vectors. If the probability vector has zero variance under the stated marginal
assumptions, then $\widehat S_n=0$, while $\rho(p,y)$ is undefined.
\end{proposition}

\noindent\emph{Proof.}
When the probability vector has nonzero variance, since
$\mathrm{var}(y)=\bar y(1-\bar y)$,
\[
\widehat S_n
=\rho(p,y)\sqrt{\frac{\mathrm{var}(p)}{\bar p(1-\bar p)}},
\]
where $\rho$ is the Pearson correlation. Because
$\mathrm{var}(p)\le\bar p(1-\bar p)$ for $p\in[0,1]$ (equality iff $p$ is
binary), $|\widehat S_n|\le|\rho(p,y)|\le1$, so
$\widehat S_n\in[-1,1]$. If the probability vector has zero variance under the
stated marginal assumptions, then
$p_i=\bar p$ for every $i$, so $\widehat{\mathrm{cov}}(p,y)=0$. The
denominator in Eq.~\eqref{eq:closed} remains positive under the stated
assumption, which gives $\widehat S_n=0$. Pearson correlation is undefined
because the probability vector has zero variance. Thus the range bound holds
throughout the stated domain, whereas the comparison with $|\rho(p,y)|$
applies only to nonconstant probability vectors.

\begin{proposition}[Consistency and asymptotic uncertainty]
Assume $(P_i,Y_i)$ are i.i.d., $0<\mu<1$, $0<\pi<1$, and
$\Sigma=\mathrm{Var}(P,Y,PY)$ is finite. Then
\[
\widehat S_n \xrightarrow{p} S(P,Y).
\]
For delta-method uncertainty, with
\[
h(a,b,c)=\frac{c-ab}{\sqrt{a(1-a)b(1-b)}},
\]
the delta method gives
\begin{align}
\sqrt n\{\widehat S_n-S(P,Y)\}
&\Rightarrow N(0,\sigma_S^2), \nonumber\\
g_\theta&=\nabla h(\theta),\qquad
\theta=(\mu,\pi,m_{PY}), \nonumber\\
\sigma_S^2&=g_\theta^\top\Sigma g_\theta.
\label{eq:delta-method}
\end{align}
\end{proposition}

\noindent\emph{Proof.}
The sample vector
$(\bar p,\bar y,n^{-1}\sum_i p_i y_i)$ converges to $\theta$ by the law of large
numbers and is asymptotically normal by the multivariate central limit theorem.
The closed form in Eq.~\eqref{eq:closed} is $h$ applied to this vector.
Continuity gives consistency; differentiability on the non-degenerate interior
gives the delta-method limit stated in Eq.~\eqref{eq:delta-method}.

\begin{proposition}[Calibration dependence]
\label{prop:calibration}
SoftMCC is not a rank-only statistic. Under a monotone recalibration
$g:[0,1]\to[0,1]$, its population value becomes
\[
S(g(P),Y)=
\frac{\mathbb E[g(P)Y]-\mathbb E[g(P)]\pi}
{\sqrt{\mathbb E[g(P)](1-\mathbb E[g(P)])\,\pi(1-\pi)}}.
\]
In general $S(g(P),Y)\ne S(P,Y)$, even when $g$ preserves every probability
rank. The temperature-scaling stress test uses
\[
g_T(p)=\sigma\{\mathrm{logit}(p)/T\},\qquad T>0,
\]
which preserves ordering but changes probability magnitudes.
\end{proposition}

In the temperature transform above, $\sigma$ denotes the logistic function.
This property explains the empirical behavior in Section~\ref{sec:results}
because AUROC and AUPRC are rank-based, while SoftMCC and the Brier score depend on
calibrated probability values. SoftMCC is threshold-free, but it should therefore be applied to calibrated probabilities rather than treated as calibration-invariant.

\begin{proposition}[Calibration-regime reduction to the Brier skill score]
\label{prop:brier}
Suppose the predicted probabilities are perfectly calibrated in the population, so that
$\mathbb E[Y\mid P]=P$. Then $\mu=\pi$, and the population core score reduces
from Eq.~\eqref{eq:population-functional} to
\[
S(P,Y)=\frac{\operatorname{Var}(P)}{\mu(1-\mu)}
=1-\frac{\mathbb E[(P-Y)^2]}{\pi(1-\pi)},
\]
the Brier skill score (BSS) of $P$ relative to the base-rate forecast $\pi$. Because
$\pi(1-\pi)$ is common to candidates evaluated on the same population, their
population ranking by SoftMCC is identical to their ranking by the negative Brier
score.
\end{proposition}

\noindent\emph{Proof.}
Under $\mathbb E[Y\mid P]=P$, $\pi=\mathbb E[Y]=\mathbb E[P]=\mu$ and
$\mathbb E[PY]=\mathbb E[P\,\mathbb E[Y\mid P]]=\mathbb E[P^2]$. The numerator
specified by Eq.~\eqref{eq:population-functional} is
$\mathbb E[PY]-\mu\pi=\mathbb E[P^2]-\mu^2=\operatorname{Var}(P)$ and the
denominator is $\sqrt{\mu(1-\mu)\,\pi(1-\pi)}=\mu(1-\mu)$. Since $Y^2=Y$, the Brier
score is $\mathbb E[(P-Y)^2]=\mathbb E[P^2]-2\mathbb E[PY]+\mathbb E[Y^2]
=\mu-\mathbb E[P^2]=\mu(1-\mu)-\operatorname{Var}(P)$, and dividing by
$\pi(1-\pi)=\mu(1-\mu)$ yields the skill-score form. Because $\pi(1-\pi)$ is common
to candidates evaluated on the same target population, $S$ is an increasing
affine function of $-\mathbb E[(P-Y)^2]$, so the two population rankings
coincide.

Proposition~\ref{prop:brier} places SoftMCC precisely between the hard-label MCC
family characterized in Proposition~\ref{prop:reduction-mcc} and probability-quality scoring. In
the calibrated regime that the framework recommends, it coincides with an
established proper scoring rule. The proposition restates a known identity for
the soft-count MCC. In the partition of the Brier score by
\citet{murphy1973partition}, perfect calibration removes the reliability term
and leaves a skill score equal to resolution divided by uncertainty, where the
resolution term equals $\operatorname{Var}(P)$. The identity gives SoftMCC and the Brier score the same candidate ordering in that regime.
The implication is one-way. Some miscalibrated joint distributions also
satisfy $S(P,Y)=\mathrm{BSS}(P,Y)$. Equality does not establish calibration, and
the gap does not identify a calibration-error functional. Data generated
from a calibrated mechanism also need not produce exact equality on a realized finite
validation sample because its empirical moments fluctuate. The empirical study
in Section~\ref{sec:results} therefore uses the equality only as a
calibrated-population reference and describes observed differences as departures
from that identity, not as a calibration diagnostic or a performance claim.

\subsection{Evaluation protocol}
\label{sec:protocol}
Metrics are assessed as model-selection criteria. For each dataset and repeat,
the data are split into training, validation, and test partitions. A calibrated
candidate pool $\mathcal F=\{f_1,\ldots,f_M\}$ is fit on training data and scored
on validation data. For a selection score $s$, define
\begin{equation}
V_s(m)=s\{y^{\mathrm{val}},p_m^{\mathrm{val}}\},
\qquad
\widehat m_s=\arg\max_{1\le m\le M} V_s(m),
\label{eq:selector}
\end{equation}
with the Brier score multiplied by $-1$ so that larger values are uniformly
better. Exact ties in $V_s$ are resolved by the earliest candidate in the fixed
declaration order, so the selector is deterministic. Ties do occur under severe
imbalance, where several candidates can attain the same validation score.
Threshold-tuned baselines (MCC@best and $F_1$@best) use the validation threshold that maximizes the corresponding score,
found by an exhaustive scan over all distinct validation probabilities with the
rule $p\ge t$, and the smallest maximizing threshold is kept.
MCC@0.5 uses a fixed validation threshold of 0.5.
SoftMCC, AUROC, AUPRC, and the Brier score are threshold-free.

The selected model is evaluated on test at a common operating rule,
\begin{equation}
U_s=
\mathrm{MCC}\!\left(
y^{\mathrm{test}},
\mathbf 1\{p_{\widehat m_s}^{\mathrm{test}}
\ge \widehat\tau_{\widehat m_s}^{\mathrm{val}}\}
\right),
\label{eq:test-utility}
\end{equation}
where $\widehat\tau_m^{\mathrm{val}}$ is the validation MCC-optimal threshold for
candidate $m$. The endpoint in Eq.~\eqref{eq:test-utility} separates the model chosen by a
metric from the threshold used to report final test utility.

For ranking stability, let $r_{qm}^{(s)}$ be the rank assigned by score $s$ to
candidate $m$ in repeat $q\in\{1,\ldots,R\}$, let
$A_m^{(s)}=\sum_q r_{qm}^{(s)}$, let $\bar A=R(M+1)/2$, and let
$T_q^{(s)}=\sum_g(t_{qg}^3-t_{qg})$, where $t_{qg}$ is the size of tie group $g$
in repeat $q$. Kendall's coefficient of concordance is
\begin{equation}
W_s=
\frac{12\sum_{m=1}^M(A_m^{(s)}-\bar A)^2}
{R^2(M^3-M)-R\sum_{q=1}^R T_q^{(s)}}.
\label{eq:kendall-w}
\end{equation}
Tied per-repeat scores receive midranks, and the denominator
in Eq.~\eqref{eq:kendall-w} adjusts the normalization for the reduction in rank dispersion caused by ties. The deterministic
fixed declaration-order tie rule in Eq.~\eqref{eq:selector} is used only to choose one model
for test utility. It does not enter $W_s$. All candidate-level validation scores,
midranks, and tie-group corrections are archived. With $M=5$ candidates and
$R=12$ repeats, $W_s$ remains coarsely quantized, so its bootstrap intervals, provided in the evidence package, are
reported as approximate. Calibration-shift sensitivity is measured by applying $g_T$ with
$T\in\{0.5,1.5,2,3\}$, covering both sharpening ($T<1$) and flattening ($T>1$)
of the probability scale, and computing Spearman agreement between the $T=1$
and shifted candidate rankings. Reported agreement values are means over this
temperature grid unless a single temperature is stated.

Paired selection
comparisons use the Wilcoxon signed-rank test and matched-pairs rank-biserial
correlation. The six prespecified SoftMCC-versus-baseline tests form one family,
with Holm adjustment controlling familywise error at $0.05$. Within-dataset
Cliff's $\delta$ is retained only as a descriptive distributional sensitivity
summary because it does not encode the block pairing. BCa bootstrap confidence
intervals are also reported, and multi-dataset rank comparisons use the Friedman
test with Nemenyi critical differences
\citep{wilcoxon1945individual,cliff1993dominance,efron1979bootstrap,
demsar2006statistical}.
\subsection*{AI-assisted tools}
During the preparation of this work, the author used generative AI tools (OpenAI Codex and Anthropic Claude Code) for English translation, language editing, code-debugging assistance, and consistency checks. The author reviewed and edited all affected material and takes full responsibility for the content of the article.

\section{Experimental Design}
\label{sec:design}

Under the shared-pool selection protocol in Section~\ref{sec:protocol}, the
experiment measures ranking stability and held-out test-MCC utility rather
than comparing raw metric values.
The study uses 18 binary tabular settings with varied class prevalence. The
original six comprise
breast cancer Wisconsin\footnote{Source: \url{https://scikit-learn.org/stable/modules/generated/sklearn.datasets.load_breast_cancer.html}.},
two controlled-imbalance synthetic sets, two credit-card fraud prevalence
variants\footnote{Source: \url{https://www.kaggle.com/datasets/mlg-ulb/creditcardfraud}.}, and IoTID20
\citep{ullah2020iotid20}. Before the expansion outcomes were computed, a
mechanical rule was fixed for the 27 datasets returned by
\texttt{imblearn.datasets.fetch\_datasets}
\citep{lemaitre2017imbalanced}: at least $1500$ rows, at least $40$ minority
positives, positive prevalence no greater than $10\%$, and no raw image, audio,
or text origin. Twelve datasets qualified, all with finite arrays. The minority
class was mapped to the positive label.

These 18 settings are not 18 fully independent data sources. The two credit-card
settings share all $492$ positive rows, the two synthetic settings use the same
generator with different seeds, and \texttt{car\_eval\_34} and
\texttt{car\_eval\_4} share the same feature matrix with different binary
targets. The primary rank tests use settings as blocks. A conservative
source-family sensitivity additionally groups each of these three pairs and
the two thyroid tasks distributed from the same source collection, giving 14
families. Kendall's $W$ is averaged within each family for each metric before
within-family midranking and recomputation of the Friedman and Nemenyi
statistics. The exact empirical inventory appears in Table~\ref{tab:datasets}.

\begin{table}[!t]
\centering
\caption{The empirical suite combines six original settings with twelve imbalanced binary tabular settings selected by the prespecified expansion rule.}
\label{tab:datasets}
\scriptsize
\setlength{\tabcolsep}{4pt}
\begin{tabular*}{\textwidth}{@{\extracolsep{\fill}}lrrrrr@{}}
\toprule
Dataset & Rows & Features & Groups & Positives & Prevalence\\
\midrule
breast cancer        & 569     & 30  & 569     & 212     & $37.3\%$\\
synth($5\%$)         & 4{,}000 & 25  & 4{,}000 & 221     & $5.5\%$\\
synth($1\%$)         & 4{,}000 & 25  & 4{,}000 & 52      & $1.3\%$\\
credit-card($1\%$)   & 49{,}200 & 29 & 48{,}631 & 492    & $1.0\%$\\
credit-card($0.5\%$) & 98{,}400 & 29 & 96{,}700 & 492    & $0.5\%$\\
IoTID20              & 40{,}000 & 79 & 28{,}285 & 2{,}561 & $6.4\%$\\
abalone              & 4{,}177 & 10  & 4{,}177 & 391     & $9.4\%$\\
car\_eval\_34        & 1{,}728 & 21  & 1{,}728 & 134     & $7.8\%$\\
car\_eval\_4         & 1{,}728 & 21  & 1{,}728 & 65      & $3.8\%$\\
coil\_2000           & 9{,}822 & 85  & 8{,}380 & 586     & $6.0\%$\\
mammography          & 11{,}183 & 6  & 7{,}849 & 260     & $2.3\%$\\
ozone\_level         & 2{,}536 & 72  & 2{,}536 & 73      & $2.9\%$\\
protein\_homo        & 145{,}751 & 74 & 144{,}968 & 1{,}296 & $0.9\%$\\
sick\_euthyroid      & 3{,}163 & 42  & 3{,}109 & 293     & $9.3\%$\\
thyroid\_sick        & 3{,}772 & 52  & 3{,}771 & 231     & $6.1\%$\\
USCrime              & 1{,}994 & 100 & 1{,}994 & 150     & $7.5\%$\\
wine\_quality        & 4{,}898 & 11  & 3{,}961 & 183     & $3.7\%$\\
yeast\_ml8           & 2{,}417 & 103 & 2{,}417 & 178     & $7.4\%$\\
\bottomrule
\end{tabular*}
\par\smallskip
\noindent\parbox{\textwidth}{\scriptsize\itshape Notes. Groups are unique exact feature-label groups used by the duplicate-safe split. Six original settings use their original sources; twelve additional settings come from the Zenodo benchmark collection exposed by imbalanced-learn. Related prevalence, generator, and target variants are not independent data sources.}
\end{table}

For every dataset and repeat, the same candidate pool of logistic
regression with $C\in\{0.1,1,10\}$ and histogram gradient boosting with maximum
depth in $\{3,6\}$ is trained. All candidates are isotonic-calibrated on training data only
(two-fold \texttt{CalibratedClassifierCV}) and implemented with scikit-learn
\citep{pedregosa2011scikit}. The calibrated probabilities are then passed to the
validation metrics, which rank candidates with the threshold policy fixed for
each baseline. Among strictly proper scoring rules, the Brier score is used as the
selector because Proposition~\ref{prop:brier} links SoftMCC to it
directly. The logarithmic score is recorded only as an auxiliary
probability-quality diagnostic.

The repeated-split design uses $12$ duplicate-safe grouped repeats per setting,
producing $216$ setting-by-repeat blocks. Each exact feature-label duplicate group
is assigned wholly to train, validation, or test, so identical rows cannot cross
split boundaries. Duplicate groups are initially split $75/25$ into a
training and validation pool and a held-out test set, and the pool is then split
$70/30$ into training and validation, giving $52.5\%/22.5\%/25\%$
train/validation/test group fractions with stratification on the group label.
Repeat $j$ uses random seed $42+j$ for $j=0,\ldots,11$.

Each block yields three
evidence layers: selected test-MCC utility, ranking reproducibility across
repeats within the dataset, and agreement under monotone calibration shift. The
statistical procedures defined in Section~\ref{sec:protocol} are applied with BCa
bootstrap $95\%$ confidence intervals over $B=2000$ resamples and with Cliff's
$\delta$ computed within each dataset before averaging. Because the twelve
repeats within a setting draw on overlapping data, the $216$ blocks are not
mutually independent. The reported
$p$-values and intervals are therefore nominal descriptive summaries in the
sense of \citet{nadeau2003inference} and are not used to establish equivalence. The analysis treats non-significant selected-utility
differences as a bounded empirical finding, not as proof of identical metrics.

Comparisons are paired by setting and repeat, holding training and calibration
fixed. Candidate rankings and thresholds use validation data only. Test labels
enter only the final utility evaluation.
The split manifest records zero training-to-validation, training-to-test, and
validation-to-test group overlap across all $216$ blocks. Duplicate-rich datasets
still contain fewer unique feature-label groups than rows, but identical rows
cannot cross a model-selection boundary.

The reported run was executed entirely in one pinned 64-bit ARM Linux environment with Python 3.12.3, NumPy 2.4.6, pandas 2.2.3, SciPy 1.16.3,
scikit-learn 1.8.0, and imbalanced-learn 0.14.2. The results reported here come
from a rerun of the same design in the same environment that replaces the
99-quantile threshold grid of the earlier run with the exhaustive threshold scan
of Section~\ref{sec:protocol} and also stores the test probabilities. The rerun
reproduces the earlier run's data, splits, validation probabilities, and
threshold-free selections exactly, so every difference between the two runs
comes from the threshold search. An informational comparison of the earlier
Windows evidence family with the first ARM run did not reproduce the ARM results
numerically, with different selected candidates and an absolute test-MCC
difference of up to $0.38$, so cross-platform values are not mixed. Every
empirical result reported in the manuscript comes from the threshold-corrected
ARM rerun, while the Windows run is retained only as an archived provenance
record. Only the repeated-split analysis in Table~\ref{tab:main} is reported,
and pilot and prototype runs are excluded from the empirical evidence. To
prevent adaptive comparison, metric definitions, the shared candidate pool,
split counts, calibration procedure, and statistical summaries remain fixed
across settings.

\section{Results}
\label{sec:results}

Ranking stability, selected-model test utility, and calibration-shift agreement
are the three empirical endpoints used to evaluate each validation score on the
same candidate pool and splits. The per-setting stability values behind the
first endpoint are plotted in Figure~\ref{fig:stability}.

\begin{figure}[h]
\centering
\begin{minipage}{\textwidth}
\centering
\includegraphics[width=119mm]{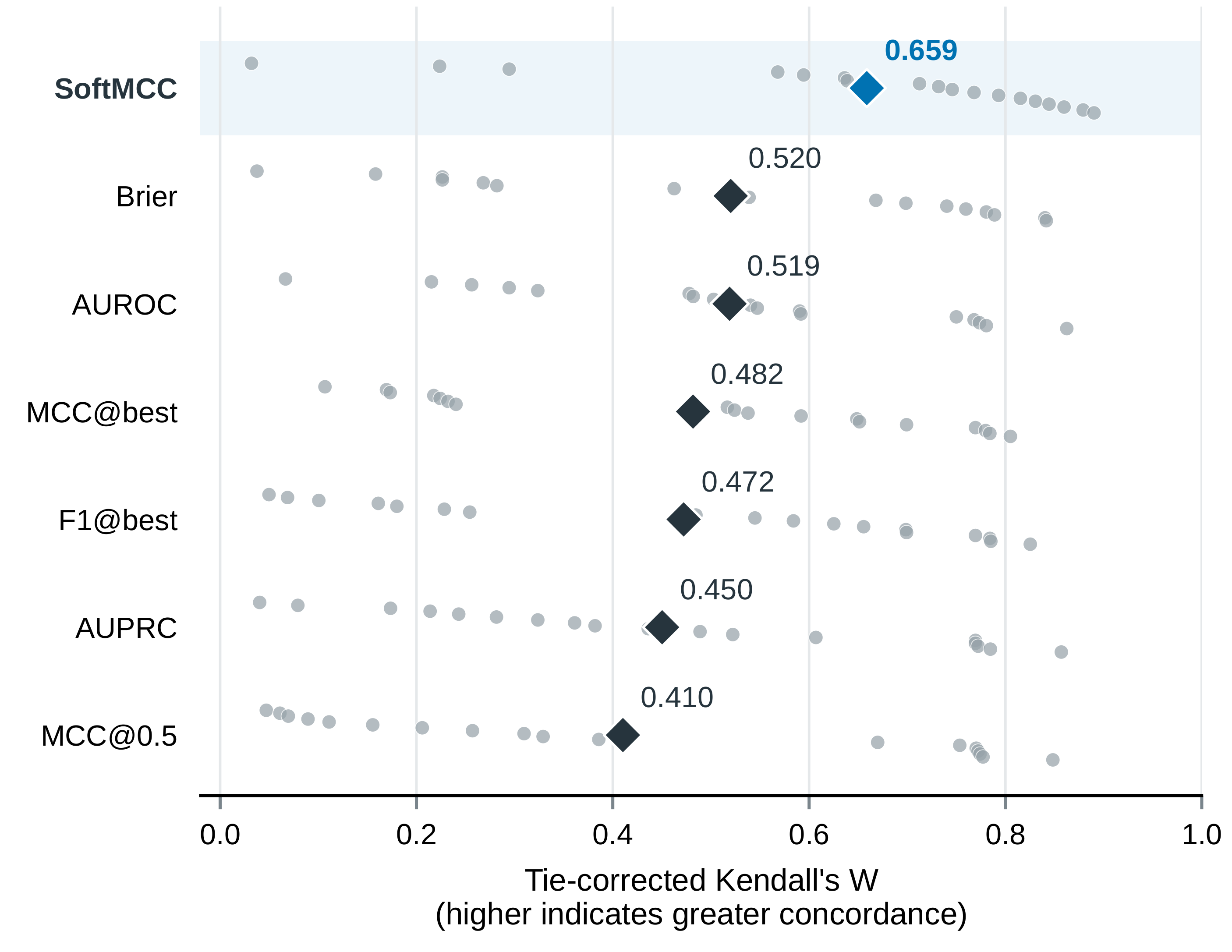}
\captionof{figure}{Across 18 settings, points show per-setting tie-corrected Kendall's $W$ and diamonds show selector means.}
\label{fig:stability}
\end{minipage}
\end{figure}

The per-setting points spread widely for every selector, so each mean
summarizes settings that differ greatly in stability. The stability result is
comparison-specific, and selected-model utility is evaluated as a separate
endpoint rather than inferred from stability.

\subsection{Paired comparisons of selected-model utility}
SoftMCC-selected models have mean test MCC $0.5793$ over the $216$ blocks,
against $0.5791$ for MCC@best, $0.5805$ for $F_1$@best, and $0.5804$ for Brier.
No paired comparison is significant at the nominal block level. The SoftMCC
differences are $+0.0002$ against MCC@best (raw Wilcoxon $p=0.792$, Holm
$p=1.000$), $-0.0012$ against $F_1$@best ($p=0.376$, Holm $p=1.000$), and
$-0.0011$ against Brier ($p=0.988$, Holm $p=1.000$). The corresponding
effective nonzero pair counts are $108$, $108$, and $59$, with matched-pairs
rank-biserial correlations $-0.029$, $-0.098$, and $0.002$.

The remaining mean differences are $+0.0011$ against
AUPRC (raw $p=0.536$, Holm $p=1.000$), $+0.0001$ against AUROC
($p=0.905$, Holm $p=1.000$), and $+0.0112$ against MCC@0.5
($p=0.093$, Holm $p=0.560$). Their effective nonzero pair counts are $104$,
$112$, and $98$, with rank-biserial correlations $0.070$, $0.013$, and
$0.195$. None of the six comparisons is significant before or after Holm
correction. The BCa $95\%$ interval for the MCC@0.5 difference,
$[0.0039,0.0201]$, excludes zero, whereas the intervals for the other five
comparisons cross zero, including $[-0.0055,0.0066]$ for MCC@best and
$[-0.0069,0.0043]$ for $F_1$@best.

These block-level tests are nominal because repeats within each setting share
observations. At the dataset level, the selection mean ranks in
Table~\ref{tab:main} place SoftMCC at $3.86$ versus $3.36$ for $F_1$@best, and
the selection Friedman test is not significant
($\chi^2=3.49$, $p=0.746$). The utility comparisons are therefore inconclusive
in both directions under the prespecified paired rule. They do not establish
equivalence, because the protocol fixes no equivalence margin, and no
dataset-level utility ordering is inferred.

\begin{table}[!tb]
\centering
\caption{The 18-setting protocol summarizes ranking stability, selected-model test utility, and calibration-shift agreement for seven validation selectors.}
\label{tab:main}
\small
\setlength{\tabcolsep}{4pt}
\begin{tabular*}{\textwidth}{@{\extracolsep{\fill}}lccccc@{}}
\toprule
 & \multicolumn{2}{c}{Stability} & \multicolumn{2}{c}{Selection utility} & Calibration\\
\cmidrule(lr){2-3}
\cmidrule(lr){4-5}
Metric & Mean $W$ & Mean rank & Mean test MCC & Mean rank & Spearman\\
\midrule
SoftMCC      & \textbf{0.659} & \textbf{2.31} & 0.5793 & 3.86 & 0.851\\
MCC@best     & 0.482 & 4.06 & 0.5791 & 4.00 & 1.000\\
$F_1$@best   & 0.472 & 4.11 & \textbf{0.5805} & \textbf{3.36} & 1.000\\
AUPRC        & 0.450 & 4.53 & 0.5782 & 4.67 & 1.000\\
AUROC        & 0.519 & 4.08 & 0.5792 & 4.11 & 1.000\\
Brier        & 0.520 & 3.67 & 0.5804 & 3.92 & 0.782\\
MCC@0.5      & 0.410 & 5.25 & 0.5681 & 4.08 & 1.000\\
\bottomrule
\end{tabular*}
\par\smallskip
\noindent\parbox{\textwidth}{\small\itshape Notes. Lower mean rank is better, and $W$ is the mean of per-setting tie-corrected Kendall concordance. Test MCC is averaged over 216 setting-by-repeat blocks; Spearman is averaged against the $T=1$ ranking over four shifted temperatures. The setting-level stability Friedman test gives $p=0.004$ with Nemenyi CD $=2.124$, the selection-utility test gives $p=0.746$, and none of the six paired comparisons is significant after Holm correction.}
\end{table}

SoftMCC and Brier choose the same candidate in $154$ of $216$ blocks, the
highest agreement with any baseline. This finite-sample agreement does not test
the population identity in Proposition~\ref{prop:brier}.

\subsection{Ranking stability across settings and source families}
The per-setting results in Figure~\ref{fig:stability}, the selector summaries
in Table~\ref{tab:main}, and the critical-difference comparison in
Figure~\ref{fig:cd} show substantial differences in the reproducibility of
candidate rankings across repeats. SoftMCC attains
the highest mean Kendall's $W$ across the 18 settings ($0.659$) and the best
stability mean rank ($2.31$). The Friedman test rejects a common rank
distribution ($\chi^2=18.84$, $p=0.0044$). With CD $=2.124$, SoftMCC is
separated from AUPRC (rank gap $2.22$) and MCC@0.5 (gap $2.94$).

\begin{figure}[h]
\centering
\begin{minipage}{\textwidth}
\centering
\includegraphics[width=119mm]{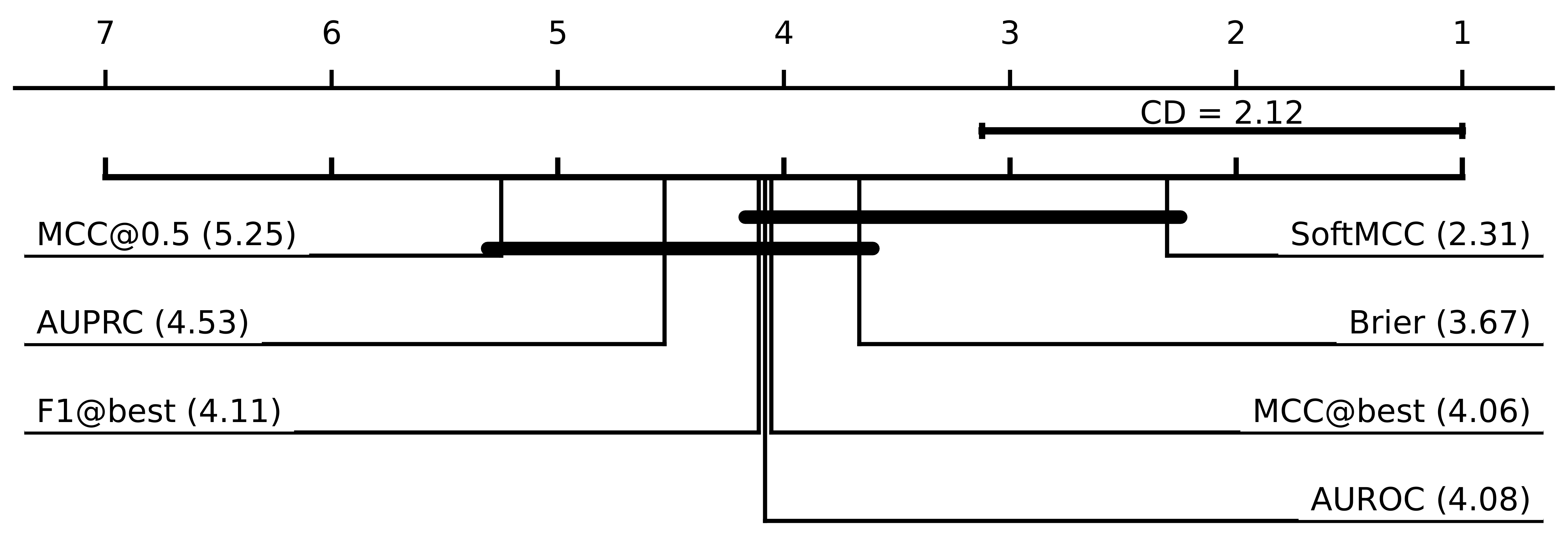}
\captionof{figure}{Across 18 settings, stability mean ranks with the Nemenyi critical difference show which selectors are separated.}
\label{fig:cd}
\end{minipage}
\end{figure}

At the prespecified significance level, the Nemenyi analysis does not distinguish
SoftMCC from Brier, MCC@best, AUROC, or $F_1$@best. SoftMCC has the highest or tied-highest per-setting $W$ in 12 of 18
settings, but it is not uniformly stable. Its $W$ is $0.224$ on
\texttt{car\_eval\_34} and $0.032$ on \texttt{wine\_quality}.

The 14-source-family sensitivity analysis defines the conservative boundary of
this result. Its Friedman test remains significant ($\chi^2=20.18$,
$p=0.0026$), with SoftMCC mean rank $2.04$ and CD $=2.408$, and SoftMCC stays
separated from AUPRC (rank gap $2.50$) and MCC@0.5 (gap $3.50$). Both Nemenyi
separations are therefore retained across the two analyses.

Three controls on the same splits test whether this stability reflects label
information or only probability sharpness. Because the closed form factorizes
as $\widehat S_n=\rho(p,y)\,\kappa(p)$ with the sharpness factor
$\kappa(p)=\sqrt{\operatorname{var}(p)/\{\bar p(1-\bar p)\}}$ depending only on
the probability vector, a high Kendall's $W$ could in principle reflect that
label-free factor rather than useful signal, and the controls rule this out as
the sole explanation. Across $200$ deterministic
validation-label permutations, SoftMCC's mean $W$ falls from $0.659$ to a null
mean of $0.092$ (central $95\%$ permutation interval $[0.058,0.129]$;
one-sided empirical exceedance $0.005$).

The factor-only controls clarify the partial mechanism. A $\kappa$-only selector remains
reproducible ($W=0.602$), and a $\rho$-only selector reaches $W=0.520$.
SoftMCC rankings agree with the $\kappa$-only and $\rho$-only rankings at mean
Spearman values $0.737$ and $0.806$, respectively. The observed stability
therefore depends on label information, while probability sharpness remains a
substantial partial mechanism and a source of calibration sensitivity.

\subsection{Ranking changes under temperature scaling}
Under temperature-scaling miscalibration over the grid
$T\in\{0.5,1.5,2,3\}$, the rankings induced by rank-based and
threshold-optimized metrics are invariant by construction (Spearman $=1.000$),
whereas SoftMCC (mean Spearman $0.851$ over all 18 settings and four
temperatures, BCa $95\%$ confidence interval $[0.824,0.872]$) and the Brier score ($0.782$,
$[0.749,0.809]$) shift, as the calibration comparison in
Figure~\ref{fig:calibration} shows. SoftMCC's mean agreement is $0.785$ at $T=3$, compared
with $0.608$ for Brier. This matches Proposition~\ref{prop:calibration}
because SoftMCC reads probability magnitudes, not only ranks, and temperature
scaling is monotone and therefore structurally favors rank-based metrics. The
test quantifies SoftMCC's calibration sensitivity rather than establishing
invariance, so the practical implication is that SoftMCC should be applied
only to calibrated probabilities.

\begin{figure}[h]
\centering
\begin{minipage}{\textwidth}
\centering
\includegraphics[width=119mm]{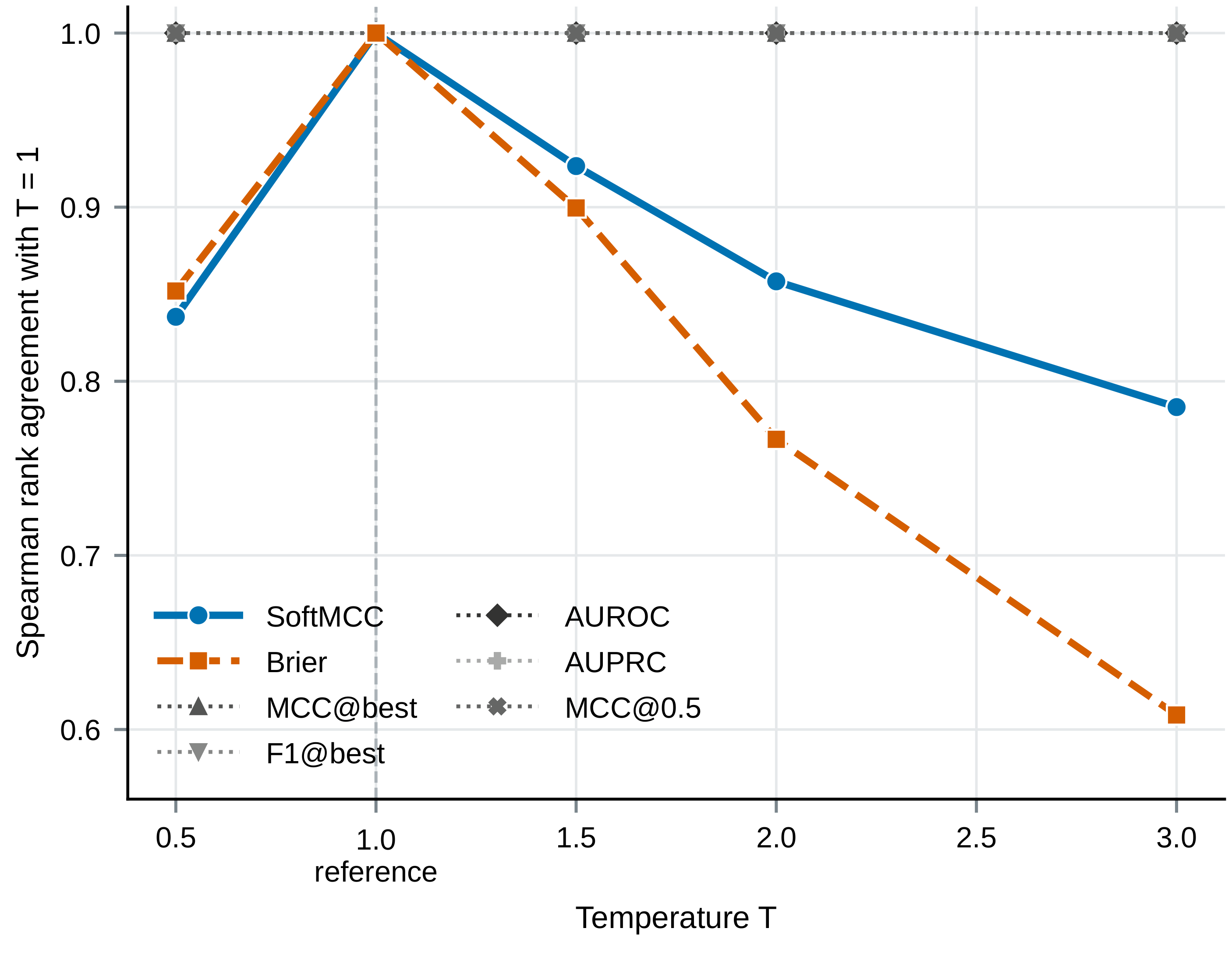}
\captionof{figure}{Across 18 settings, candidate-rank agreement with $T=1$ shows temperature sensitivity for SoftMCC and Brier and invariance for the other selectors.}
\label{fig:calibration}
\end{minipage}
\end{figure}

Alongside this calibration requirement, the primary endpoints distinguish
ranking stability from selected-model utility, as the two mean-rank columns of
Table~\ref{tab:main} show. SoftMCC ranks first for stability and second for
selection utility, and post-hoc separations are drawn in Figure~\ref{fig:cd}
only for stability because the selection-utility omnibus test is not
significant. Read together, the two columns locate the comparison-specific
stability result and show that it does not translate into a demonstrated
selected-utility advantage in the evaluated pool.

\section{Discussion}
\label{sec:discussion}

SoftMCC connects probability-valued MCC to proper scoring and repeated model
selection. Probability-valued confusion matrices already support post-training
evaluation \citep{wang2013pcen,yacouby2020probabilistic,aguilarruiz2024certainty},
while differentiable metric formulations use related constructions during
training \citep{abhishek2021mcc,tsoi2022bridging,han2024anyloss}. The present
contribution concerns the MCC-specific covariance form, hard-label reduction,
and the soft-count statement of the calibrated Brier identity known from the
Murphy partition, together with evaluation of the resulting
validation selector on a shared candidate pool. This distinction locates the
theoretical result within the established evaluation literature and separates
the score's algebraic properties from its observed model-selection behavior.
The empirical comparisons assess that behavior without attributing a new
probabilistic-count kernel to SoftMCC.

The calibrated identity clarifies how SoftMCC relates to probability quality.
Proper scoring rules assess probabilistic forecasts against observed outcomes
\citep{gneiting2007strictly}, whereas calibration describes agreement between
predicted probabilities and observed frequencies \citep{silva2023calibration}.
Under perfect population calibration, SoftMCC equals the Brier skill score and
therefore gives the same candidate ordering as negative Brier on a fixed
population. As noted after Proposition~\ref{prop:brier}, equality does not
imply calibration. Consequently, frequent agreement with Brier in selected
candidates is compatible with the theory but cannot verify its population
assumption. A finite-sample gap additionally contains sampling fluctuation, so
neither equality nor the gap supplies a calibration diagnostic.

Recent calibration research further motivates separating these questions.
\citet{dong2025survey} examine the difficulties of estimating reliable
probabilities under class imbalance. In transformer-based text classification,
\citet{fonseca2026high} show that aggregate calibration summaries can obscure
errors among incorrect predictions and propose a Balanced Brier Score for
that setting. These studies address calibration assessment in different model
families and do not establish the SoftMCC identity. Their relevance here is
that a favorable aggregate score requires interpretation against its target
property. The identity uses the ordinary Brier skill score under an explicit
population assumption, rather than the Balanced Brier Score or a guarantee
that fitted candidates are calibrated.

The stability and utility findings answer different evaluation questions.
This distinction is consistent with the context-dependent metric comparisons
of \citet{gaudreault2024empirical}, which motivate using measures with different
roles in imbalanced classification. SoftMCC has the best mean stability rank
in the evaluated pool. The primary Nemenyi analysis separates it from AUPRC
and MCC@0.5, and both separations are retained under source-family
aggregation. Selected-model utility has a different pattern. The six paired
mean differences are small and of mixed sign, none is significant before or
after Holm correction, and the dataset-level test is not significant.
The practical interpretation therefore depends on both reproducibility of
the validation ranking and test utility of the resulting choice.

How far these comparisons can be generalized depends on the inferential unit.
Repeated-data dependence is a recognized issue in estimating generalization
error \citep{nadeau2003inference}, and multi-dataset rank comparisons operate
at the dataset level \citep{demsar2006statistical}. Here, paired tests preserve
repeatwise matching, but overlapping observations make block-level $p$-values
nominal. Setting-level tests treat benchmark variants as separate blocks,
whereas source-family aggregation reduces dependence among related prevalence,
generator, target, and thyroid variants. Both units give the same separations
here, and the family-level analysis remains the one that bounds the
generalization claim. Grouped splitting prevents exact
feature-label duplicates from crossing partitions, but neither that protection
nor family aggregation creates independent external cohorts.

Ranking concordance also needs an interpretation distinct from predictive
performance, in keeping with the separation of evaluation targets and test
assumptions discussed by \citet{rainio2024evaluation}. Kendall's $W$ summarizes
agreement of complete candidate orderings across repeats. It does not quantify
the probability that the first-ranked candidate has the best deployment
utility. A modest pool can be ordered consistently even when selected models
have small or unfavorable utility differences. Tie correction accounts for
equal scores in measuring concordance, while the common test-MCC endpoint
evaluates the consequence of selection. Reading these endpoints together
allows an application to assess whether the observed ranking consistency is
useful for its operating objective.

The mechanism controls help explain this distinction without changing the
utility evidence. Calibration and probability sharpness are separate concepts
in probabilistic forecast assessment \citep{gneiting2007strictly,silva2023calibration}.
The factorization in Section~\ref{sec:method} separates probability-label
association from a label-free probability-sharpness factor. Label permutation
largely removes ranking reproducibility, whereas the separate $\rho$ and
$\kappa$ selectors each retain part of the ordering consistency. Thus,
sharpness contributes to the observed stability but does not explain it alone.
These controls characterize the behavior of the validation rankings and leave
candidate calibration unmeasured, while selected-model test utility remains a
separate endpoint.

Temperature scaling makes the probability-scale dependence operationally
visible. Its use in post-hoc calibration \citep{guo2017calibration} and the
broader calibration literature \citep{silva2023calibration,dong2025survey}
distinguishes changes in confidence from changes in class ordering. In the
controlled experiment, strictly monotone transformations leave rank-based
and threshold-optimized rankings invariant, while SoftMCC and Brier rankings
change. This agrees with the probability-magnitude dependence established in
Proposition~\ref{prop:calibration}. The experiment measures sensitivity to
the specified transformations, rather than robustness to arbitrary
distribution shift. It also does not show that one calibration method would
produce the same selections as another, because the fitted candidates use
the same isotonic-calibration protocol throughout.

Reproducible comparison requires both a defined benchmark and a traceable
execution environment. Benchmark resources such as PMLB facilitate consistent
evaluation across datasets \citep{olson2017pmlb}, whereas execution-environment
effects such as the platform divergence observed here are documented at the
run level. The earlier Windows run did not reproduce the ARM results closely
enough for pooling, so all reported evidence comes from the pinned ARM
environment and the divergence is recorded rather than averaged away.
Manifests, environment versions, execution records, and artifact hashes make
the reported run auditable. Attributing the divergence to a software,
numerical, or execution component, and deciding whether either platform should
be treated as authoritative, is left to a controlled cross-platform
reproduction.

Practical metric choice should follow the decision being made. Cost-sensitive
evaluation and unified views of classification performance explicitly relate
scores to operating conditions \citep{drummond2006cost,hand2009hmeasure,hernandez2012unified}.
SoftMCC is applicable when validation calls for a threshold-free ranking among
calibrated probabilistic candidates and stable rankings are valued even though
no selected-utility gain was observed. A fixed operating threshold, explicit error costs,
or unreliable calibration can instead favor threshold-optimized MCC, cost
curves, expected-loss analysis, rank-based measures, or proper scores.
Reporting the final operating-threshold utility remains necessary in the
terms of the application. The choice of a validation selector does not by
itself determine the appropriate deployment decision.

These practical recommendations should be interpreted within the benchmark's
scope, because the datasets, candidate models, and calibration conditions
determine which aspects of model selection the evidence can address. Recent analyses
continue to show that classification-metric behavior depends on class
imbalance \citep{vanzyl2025analysis}, while calibration studies examine
different model and task regimes \citep{dong2025survey,fonseca2026high}.
The evidence covers calibrated binary tabular selection, a moderate candidate
pool, related settings, synthetic prevalence shifts, and monotone temperature
scaling. Multiclass or structured outputs, segmentation training, online drift,
cost-sensitive deployment, and uncalibrated models require separate evidence.
Independent external cohorts and broader candidate pools would test the
generalizability of the observed ordering behavior. These extensions should retain
an explicit utility endpoint so that improved ranking reproducibility can
be interpreted alongside the consequences of model selection.

\section{Conclusion}
\label{sec:conclusion}

SoftMCC links probability-valued MCC theory to a threshold-free validation
protocol built on established probabilistic confusion counts. The core score
has a covariance form, reduces to hard MCC, and coincides with the Brier skill
score under perfect population calibration. Because the converse fails, this
identity is a model-selection bridge rather than a calibration diagnostic.
Empirically, SoftMCC gives the highest mean ranking reproducibility among the
seven selectors in the evaluated pool, with Nemenyi separation from AUPRC and
MCC@0.5 retained under source-family aggregation, while differences in
selected-model utility are small and not significant. It is therefore suited
to validation settings in which calibrated probabilities and reproducible
selection matter and a utility gain from the selector is not required, and
it informs validation without replacing task-specific operating decisions.

Building on this distinction between validation selection and deployment
utility, future work can examine how the framework behaves when either the
selection setting or the practical decision requirements change. Extensions to
cost-sensitive, prevalence-adjusted, net-benefit, temporally drifting, or
multiclass settings change the validation selector or the deployment question
itself, so each belongs to a separate study with its own prespecified
candidate space, calibration method, split construction, and utility endpoint.
Nearer to the present target, independent external cohorts and broader
candidate pools would test stability beyond related benchmark settings.
Controlled cross-platform reproduction would resolve the divergence that
confined these results to one execution family, while varying the calibration
method at a fixed utility endpoint would separate calibration sensitivity from
selector behavior. 

\FloatBarrier
\section*{Statements and Declarations}

\subsection*{Data and code availability}
The scorer implementation and the complete 18-setting ARM evidence package,
including the threshold-corrected rerun reported here, are
publicly available at \url{https://github.com/canay/softmcc_theory}.
The repository
contains source links for all public datasets, documents the raw-data sharing
boundary, and provides grouped-split experiment and analysis code, preprocessing
instructions, environment versions, split and dataset manifests, execution
records, artifact hashes, and derived summaries used for the manuscript tables
and figures. Raw third-party records are not redistributed. Their acquisition
and reuse remain governed by the source repositories' terms.

\subsection*{Declaration of competing interest}
The author declares no competing interests.

\subsection*{Funding}
This research did not receive any specific grant from funding agencies in the
public, commercial, or not-for-profit sectors.

\subsection*{CRediT authorship contribution statement}
\"Ozkan CANAY: Conceptualization, Methodology, Software, Validation, Formal
analysis, Investigation, Data curation, Writing -- original draft,
Writing -- review and editing, and Visualization.

\bibliography{references}

\end{document}